\documentclass{article}
\usepackage{ijcai22}

\usepackage{times}
\usepackage{soul}
\usepackage{url}
\usepackage[hidelinks]{hyperref}
\usepackage[utf8]{inputenc}
\usepackage[small]{caption}
\usepackage{graphicx}
\usepackage{amsmath}
\usepackage{amsthm}
\usepackage{booktabs}
\usepackage{algorithm}
\usepackage{algorithmic}
\usepackage{lineno,hyperref}
\usepackage{amsmath}
\usepackage{amsfonts}
\usepackage{algorithm}
\usepackage{algorithmic}
\usepackage{array}
\usepackage{graphicx}
\usepackage{subfigure}
\usepackage{booktabs}
\usepackage{multirow}
\usepackage{subfigure}
\usepackage{caption}
\usepackage{graphicx} 
\usepackage{makecell}
\usepackage{hyperref}
\usepackage{amsmath}
\usepackage{amsthm}
\usepackage{amssymb}
\PassOptionsToPackage{hyphens}{url}\usepackage{hyperref}

\newcommand{\argmin}{\operatornamewithlimits{arg\,min}}
\newcommand{\xx}{\boldsymbol x}
\newcommand{\yy}{\boldsymbol y}

\newcommand{\Ltwo}{{L^2(\mu)}}

\newtheorem{theorem}{Theorem}

\newtheorem{remark}{Remark}
\newtheorem{corollary}{Corollary}

\newtheorem{assumption}{Assumption}

\newcommand{\Nystrom}[1]{{Nystr\"om}}

\newcommand{\XX}{{\boldsymbol X}}

\newcommand{\RR}{\mathbb{R}}

\newcommand{\bpr}{\begin{proof}}
		\newcommand{\epr}{\end{proof}}
\newcommand{\be}{\begin{equation}}
		\newcommand{\ee}{\end{equation}}

\newcommand{\mh}{{\mathcal{H}_K}}
\newcommand{\mO}{\mathcal{O}}
\newcommand{\mL}{\mathcal{L}}
\newcommand{\mE}{\mathcal{E}}

\newcommand{\frho}{f^*}

\newcommand{\fkrr}{\widehat{f}}
\newcommand{\fridgekrr}{\widehat{f}_\lambda}
\newcommand{\frrrf}{\widehat{f}_M}
\newcommand{\frrsgd}{\widehat{f}_{M, b, t}}
\newcommand{\frrgd}{\widehat{f}_{M, t}}

\newcommand{\mX}{\mathcal{X}}
\newcommand{\mY}{\mathcal{Y}}

\def\OO{{\boldsymbol \Omega}}

\def\bb{{\boldsymbol b}}
\def\ww{{\boldsymbol w}}
\def\vv{{\boldsymbol v}}
\def\xx{{\boldsymbol x}}
\def\yy{{\boldsymbol y}}

\def\oo{{\boldsymbol \omega}}

\title{Ridgeless Regression with Random Features}

\author{
Jian Li$^1$\and
Yong Liu$^2$\footnote{Corresponding author}\and
Yingying Zhang$^3$\\
\affiliations
$^1$Institute of Information Engineering, Chinese Academy of Sciences\\
$^2$Gaoling School of Artificial Intelligence, Renmin University of China\\
$^3$School of Mathematics and Information Sciences, Yantai University\\
\emails
lijian9026@iie.ac.cn, 
liuyonggsai@ruc.edu.cn,
yingyingzhang239@gmail.com
}

\begin{document}

\maketitle
\begin{abstract}
    Recent theoretical studies illustrated that kernel ridgeless regression can guarantee good generalization ability without an explicit regularization. In this paper, we investigate the statistical properties of ridgeless regression with random features and stochastic gradient descent. We explore the effect of factors in the stochastic gradient and random features, respectively. Specifically, random features error exhibits the double-descent curve. Motivated by the theoretical findings, we propose a tunable kernel algorithm that optimizes the spectral density of kernel during training. Our work bridges the interpolation theory and practical algorithm.
\end{abstract}

\section{Introduction}
In the view of traditional statistical learning, an explicit regularization should be added to the nonparametric learning objective \cite{caponnetto2007optimal,li2018multi}, i.e. an $2$-norm penalty for the kernelized least-squares problems, known as kernel ridge regression (KRR). To ensure good generalization (out-of-sample) performance, the modern models should choose the regularization hyperparameter $\lambda$ to balancing the bias and variance, and thus avoid overfitting.
However, recent studies empirically observed that neural networks still can interpolate the training data and generalize well when $\lambda=0$ \cite{zhang2021understanding,belkin2018understand}.
But also, for many modern models including neural networks, random forests and random features, the test error captures a 'double-descent' curve as the increase of features dimensional \cite{mei2019generalization,advani2020high,nakkiran2021deep}.
Recent empirical successes of neural networks prompted a surge of theoretical results to understand the mechanism for good generalization performance of interpolation methods without penalty where researchers started with the statistical properties of ridgeless regression based on random matrix theory \cite{liang2020just,bartlett2020benign,jacot2020implicit,ghorbani2021linearized}.

However, there are still some open problems to settle down: 1) In the theoretical front, current studies focused on the direct bias-variance decomposition for the ridgeless regression, but ignore the influence from the optimization problem, i.e. stochastic gradient descent (SGD). Meanwhile, the connection between kernel regression and ridgeless regression with random features are still not well established. Further, while asymptotic behavior in the overparameterized regime is well studied, ridgeless models with a ﬁnite number of features are much less understood.
2) In the algorithmic, there is still a great gap between statistical learning for ridgeless regression and algorithms. 
Although the theories explain the double-descent phenomenon for ridgeless methods well, a natural question is whether the theoretical studies helps to improve the mainstream algorithms or design new ones.

In this paper, we consider the Random Features (RF) model \cite{rahimi2007random} that was proposed to approximate kernel methods for easing the computational burdens. In this paper, we investigate the generalization ability of ridgeless random features, of which we explore the effects from stochastic gradient algorithm and random features. And then, motivated by the theoretical findings, we propose a tunable kernel algorithm that optimizes the spectral density of kernel during training, reducing multiple trials for kernel selection to just training once.
Our contributions are summarized as:

\textbf{1) Stochastic gradients error.} We first investigate the stochastic gradients error influenced the factors from SGD, i.e. the batch size $b$, the learning rate $\gamma$ and the iterations $t$. 
    The theoretical results illustrate the tradeoffs among these factors to achieve better performance, such that it can guide the set of these factors in practice.

\textbf{2) Random features error.} We then explore the difference between ridgeless kernel predictor and ridgeless RF predictor. In the overparameterized setting $M > n$, the ridgeless RF converges to the ridgeless kernel as the increase number of random features, where $M$ is the number of random features and $n$ is the number of examples.
    In the underparameterized regime $M \leq n$, the error of ridgeless RF also exhibit an interesting "double-descent" curve in Figure \ref{fig.exp2_rf_exp3} (a), because the variance term explores near the transition point $M = n$.

\textbf{3) Random features with tunable kernel algorithm.}
    Theoretical results illustrate the errors depends on the trace of kernel matrix, motivating us to design a kernel learning algorithm which asynchronously optimizes the spectral density and model weights. 
    The algorithm is friend to random initialization, and thus easing the problem of kernel selection.

\section{Related Work}
\textbf{Statistical Properties of Random features.}
The generalization efforts for random features are mainly in reducing the number of random features to achieve the good performance.
\cite{rahimi2007random} derived the appropriate error bound between kernel function and the inner product of random features. And then, the authors proved $\mO(\sqrt{n})$ features to obtain the error bounds with convergence rate $\mO(1/\sqrt{n})$ \cite{rahimi2008weighted}.
Rademacher complexity based error bounds have been proved in \cite{li2019multi,li2020automated}.
Using the integral operator theory, \cite{rudi2017generalization,li2022optimal} proved the minimax optimal rates for random features based KRR.

In contrast, recent studies \cite{hastie2019surprises} make efforts on the overparameterized case for random features to compute the asymptotic risk and revealed the double-descent curve for random ReLU \cite{mei2019generalization} features and random Fourier features \cite{jacot2020implicit}.

\noindent\textbf{Double-descent in Ridgeless Regression.}
The double-descent phenomenon was first observed in multilayer networks on MNIST dataset for ridgeless regression \cite{advani2020high}.
It then been observed in random Fourier features and decision trees \cite{belkin2018understand}.
\cite{nakkiran2021deep} extended the double-descent curve to various models on more complicated tasks.
The connection between double-descent curve and random initialization of the Neural Tangent Kernel (NTK) has been established \cite{geiger2020scaling}.

Recent theoretical work studied the asymptotic error for ridgeless regression with kernel models \cite{liang2020just,jacot2020implicit} or linear models \cite{bartlett2020benign,ghorbani2021linearized}.

\section{Problem setup}
In the context of nonparametric supervised learning, given a probability space $\mX \times \mY$ with an unknown distribution $\mu(\xx, y)$, the regression problem with squared loss is to solve
\begin{align}
    \label{eq.expected_loss}
    \min_f \mathcal{E}(f), \qquad \mathcal{E}(f) = \int_{\mX \times \mY} (f(\xx) - y)^2 d \mu(\xx, y).
\end{align}
However, one can only observe the training set $(\xx_i, y_i)_{i=1}^n$ that drawn i.i.d. from $\mX \times \mY$ according to $\mu(\xx, y)$, where $\xx_i \in \RR^d$ are inputs and $y_i \in \RR$ are the corresponding labels.

\subsection{Kernel Ridgeless Regression.}

Suppose the target regression $\frho(\xx)=\mathbb{E}(y|\xx = x)$ lie in a Reproducing Kernel Hilbert Space (RKHS) $\mh$, endowed with the norm $\|\cdot\|_K$ and Mercer kernel $K(\cdot, \cdot): \mX \times \mX \to \RR$.
Denote $\XX = [\xx_1, \cdots, \xx_n]^\top \in \RR^{n \times d}$ the input matrix and $\yy = [y_1, \cdots, y_n]^\top$ the response vector.
We then let $K(\XX, \XX) = [K(\xx_i, \xx_j)]_{i,j=1}^n \in \RR^{n \times n}$ be the kernel matrix and $K(\xx, \XX) = [K(\xx, \xx_1), \cdots, K(\xx, \xx_n)] \in \RR^{1 \times n}$.

Given the data $(\XX, \yy)$, the empirical solution to \eqref{eq.expected_loss} admits a closed-form solution:
\begin{align}
    \label{eq.krr}
    \fkrr(\xx) = K(\xx, \XX) K(\XX, \XX)^{\dagger} \yy,
\end{align}
where $\dagger$ is the Moore-Penrose pseudo inverse.
The above solution is known as kernel ridgeless regression.


\subsection{Ridegeless Regression with Random Features}
The Mercer kernel is the inner product of feature mapping in $\mh$, stated as $K(\xx, \xx') = \langle K_\xx, K_{\xx'}\rangle, ~ \forall \xx, \xx' \in \XX$, where $K_{\xx} = K(\xx, \cdot) \in \mh$ is high or infinite dimensional.

The integral representation for kernel is $K(\xx, \xx') = \int_\mX \psi(\xx, \oo) \psi(\xx', \oo) d \pi(\oo)$, where $\pi(\oo)$ is the spectral density and $\psi: \mX \times \mX \to \RR$ is continuous and bounded function.
Random features technique is proposed to approximate the kernel by a finite dimensional feature mapping 
\begin{equation}
    \begin{aligned}
        \label{eq.rrrf}
        &K(\xx, \xx') \approx \langle \phi(\xx), \phi(\xx') \rangle, ~ \text{with} \\
        &\phi(\xx) = \frac{1}{\sqrt{M}} \left(\psi(\xx, \oo_1), \cdots, \psi(\xx, \oo_M)\right),
    \end{aligned}
\end{equation}
where $\phi: \mX \to \RR^M$ and $\oo_1, \cdots, \oo_M$ are sampled independently according to $\pi$.
The solution of ridgeless regression with random features can be written as 
\begin{align}
    \frrrf(\xx) = \phi(\xx) \left[\phi(\XX)^\top \phi(\XX)\right]^\dagger \phi(\XX)^\top \yy,
\end{align}
where $\phi(\XX) \in \RR^{n \times M}$ is the feature mapping matrix over $\XX$ and $\phi(\xx) \in \RR^{1 \times M}$ is the feature mapping over $\xx$.


\subsection{Random Features with Stochastic Gradients}
To further accelerate the computational efficiency, we consider the stochastic gradient descent method as bellow
\begin{equation}
    \begin{aligned}
        \label{eq.rrsgd}
        & \frrsgd(\xx) = \langle \ww_t, \phi(\xx) \rangle, ~ \text{with} \\
        & \ww_{t+1} = \ww_t - \frac{\gamma_t}{b} \sum_{i=b(t-1)+1}^{bt} \langle \ww_t, \phi(\xx_i) \rangle - y_i) \phi(\xx_i),
    \end{aligned}
\end{equation}
where $\ww_t \in \RR^M$, $\ww_0 = 0$, $b$ is the mini-batch size and $\gamma_t$ is the learning rate. 
When $b=1$ the algorithm reduces to SGD and $b > 1$ is the mini-batch version.
We assume the examples are drawn uniformly with replacement, by which one pass over the data requires $\lceil n/b\rceil $ iterations.

Before the iteration, the compute of $\phi(\XX)$ consumes $\mO(nM)$ time.
The time complexity is $\mO(Mb)$ for per iteration and $\mO(MbT)$ after $T$ iterations.
Thus, the total complexities are $\mO(nM)$ for the one pass case and $\mO(MbT)$ for the multiple pass case, respectively.

\section{Main Results}
\label{sec.theory}
In this section, we study the statistical properties of estimators $\frrsgd$ \eqref{eq.rrsgd}, $\frrrf$ \eqref{eq.rrrf} and $\fkrr$ \eqref{eq.krr}.
Denote $\mathbb{E}_\mu [\cdot]$ the expectation w.r.t. the marginal $\xx \sim \mu$ and 
\begin{align*}
    \|g\|_\Ltwo^2 = \int g^2(\xx) d \mu(\xx) = \mathbb{E}_\mu [g^2(\xx)], \quad \forall g \in \Ltwo.
\end{align*}
The squared integral norm over the space $\Ltwo = \{g: \mX \to \RR | \int g^2(\xx) d \mu(\xx) < \infty\}$ and $\frho \in \Ltwo$.
Combing the above equation with \eqref{eq.expected_loss}, one can prove that $\mE(f) - \mE(\frho) = \|f - \frho\|_\Ltwo^2, \forall f \in \Ltwo$.
Therefore
we can decompose the excess risk of $\mE(\frrsgd) - \mE(\frho)$ as bellow
\begin{equation}
    \begin{aligned}
        \label{eq.error_decomposition}
        &\mE(\frrsgd) - \mE(\frho) \\
        \leq & ~\|\frrsgd - \frrrf\| + \|\frrrf - \fkrr\| + \|\fkrr - \frho\|.
    \end{aligned}
\end{equation}
The excess risk bound includes three terms: stochastic gradient error $\|\frrsgd - \frrrf\|$, random feature error $\|\frrrf - \fkrr\|$, and excess risk of kernel ridgeless regression $\|\fkrr - \frho\|$, which admits the bias-variance form. 
In this paper, since the ridgeless excess risk $\|\fkrr - \frho\|$ has been well-studied \cite{liang2020just}, we focus on the first two bounds and explore the factors in them, respectively. Throughout this paper, we assume the true regression $\frho(\xx) = \langle \frho, K_{\xx} \rangle$ lies in the RKSH of the kernel $K$, i.e., $\frho \in \mh$.

\begin{assumption}[Random features are continuous and bounded]
    \label{asm.rf}
    Assume that $\psi$ is continuous and there is a $\kappa \in [1, \infty)$, such that $|\psi(\xx, \omega)| \leq \kappa, \forall \xx \in \mathcal{X}, \omega \in \Omega$.
\end{assumption}

\begin{assumption}[Moment assumption]
    \label{asm.moment}
    Assume there exists $B > 0$ and $\sigma > 0$, such that for all $p \geq 2$ with $p \in \mathbb{N}$,
    \begin{align}
        \label{eq.moment}
        \int_\mathbb{R} |y|^p d \rho(y|\xx) \leq \frac{1}{2} p ! B^{p-2} \sigma^2.
    \end{align}
\end{assumption}

The above two assumptions are standard in statistical learning theory \cite{smale2007learning,caponnetto2007optimal,rudi2017generalization}.
According to Assumption \ref{asm.rf}, the kernel $K$ is bounded by $K(\xx, \xx) \leq \kappa^2$.
The moment assumption on the output $y$ holds when $y$ is bounded, sub-gaussian or sub-exponential.
Assumptions \ref{asm.rf} and \ref{asm.moment} are standard in the generalization analysis of KRR, always leading to the learning rate $\mathcal{O}(1/\sqrt{N})$ \cite{smale2007learning}.

\subsection{Stochastic Gradients Error}
We first investigate the approximation ability of stochastic gradient by measuring $\|\frrsgd - \frrrf\|_\Ltwo^2$, and explore the effect of the mini-batch size $b$, the learning rate $\gamma$ and the number of iterations $t$.

\begin{theorem}[Stochastic gradient error]
    \label{thm.sgd}
    Under Assumptions \ref{asm.rf}, \ref{asm.moment}, let $t \in [T]$, $\gamma \leq \frac{n}{9T\log\frac{n}{\delta}} \wedge \frac{1}{8(1 + \log T)}$ and $n \geq 32 \log^2 \frac{2}{\delta}$, the following bounded holds with high probability
    \begin{align*}
        \|\frrsgd - \frrrf\| \lesssim \frac{\gamma}{b} + \frac{\|\frho\|_K}{\gamma t}.
    \end{align*}
\end{theorem}

The first term in the above bound measures the similarity between mini-batch gradient descent estimator and full gradient descent estimator $\|\frrsgd - \frrgd\|$, which depends on the mini-batch size $b$ and the learning rate $\gamma$.
The second term reflects the approximation between the gradient estimator and the random feature estimator $\| \frrgd - \frrsgd\|$, which is determined by the number of iterations $t$ and the step-size $\gamma$, leading to a sublinear convergence $\mO(1/t)$.

\begin{corollary}
    Under the same assumptions of Theorem \ref{thm.sgd}, one of the following cases and the time complexities
    \begin{itemize}
        \item[1)] $b=1, \gamma \simeq \frac{1}{\sqrt{n}}$ and $T=n$ $\quad\Rightarrow\quad$ $\mO(nM)$ 
        \item[2)] $b=\sqrt{n}, \gamma \simeq 1$ and $T=\sqrt{n}$ $\quad\Rightarrow\quad$ $\mO(nM)$ 
        \item[3)] $b=n, \gamma \simeq 1$ and $T=\sqrt{n}$ $\quad\Rightarrow\quad$ $\mO(n\sqrt{n}M)$ 
    \end{itemize}
    is sufficient to guarantee with high probability that
    \begin{align*}
        \|\frrsgd - \frrrf\| \lesssim \frac{1}{\sqrt{n}}.
    \end{align*}
\end{corollary}

In the above corollary, we give examples for SGD, mini-batch gradient, and full gradient, respectively.
It shows the computational efficiency of full gradient descent is usually worse than stochastic gradient methods.
The computational complexities $\mO(nM)$ are much smaller than that of random features $\mO(nM^2 + M^3)$.
All these cases achieve the same learning rate $\mO(1/\sqrt{n})$ as the exact KRR.
With source condition and capacity assumption in integral operator literature \cite{caponnetto2007optimal,rudi2017generalization}, the above bound can achieve faster convergence rates.

\begin{remark}
    \cite{carratino2018learning} also studied the approximation of mini-batch gradient descent algorithm, but the random features estimator is defined with noise-free labels $\frho(\XX)$ and with ridge regularization, which failed to directly capture the effect of stochastic gradients.
    Note that this work provides empirical estimators $\frrsgd$, $\frrrf$, and $\fkrr$ with noise labels $\yy$ and ridgeless, which makes the proof techniques quite different from \cite{carratino2018learning}.
    The technical difference can be found by comparing the proofs of Theorem \ref{thm.sgd} in this paper and Lemma 9 in \cite{carratino2018learning}.
\end{remark}

\subsection{Random Features Error}
The redgeless RF predictor \eqref{eq.rrrf} characterizes different behaviors depending on the relationship between the number of random features $M$ and the number of samples $n$.
\begin{theorem}[Overparameterized regime $M \geq n$]
    \label{thm.rf_overparameterized}
    When $M \geq n$, the random features error can be bounded by
    \begin{align*}
        \mathbb{E} ~ \mE(\frrrf) - \mE(\fkrr) \lesssim  \frac{(\alpha + c_1) \|\frho\|^2_K}{M},
    \end{align*}
    where the constant $\alpha \propto \frac{M}{M - n}$.
\end{theorem}

In the overparameterized case, the ridgeless RF predictor is an unbiased estimator of the ridgeless kernel predictor and RF predictors can interpolate the dataset.
The error term in the above bound is the variance of the ridgeless RF predictor.
As shown in Figure \ref{fig.exp1_stochastic} (a), the variance of ridgeless RF estimator explodes near the interpolation as $M \to n$, leading to the double descent curve that has been well studied in \cite{mei2019generalization,jacot2020implicit}.

Obviously, a large number of random features in the overparameterized domain can approximate the kernel method well \cite{rahimi2007random,rudi2017generalization}, but it introduces computational challenges, i.e. $\mO(nM)$ time complexity for $M > n$.
To reach good tradeoffs between statistical properties and computational cost, it is rather important to investigate the approximation of ridgeless RF predictor in the underparameterized regime.

\begin{theorem}[Underparameterized regime $M < n$]
    \label{thm.rf_underparameterized}
    When $M < n$, under Assumption \ref{asm.rf}, the ridgeless RF estimator $\frrrf$ approximates a kernel ridge regression estimator $\fridgekrr = K(\xx, \XX) \left(\mathbf{K} + \lambda n I\right)^{-1} \yy$ by
    \begin{align*}
        \mathbb{E}~ \mE(\frrrf) - \mE(\fridgekrr) \lesssim  \frac{\text{Tr}(\mathbf{K})^4 \|\frho\|_K^2}{M^6} + \frac{(\alpha + c_1) \|\frho\|^2_K}{M},
    \end{align*}
    where $\mathbf{K} = K(\XX, \XX)$ is the kernel matrix, $\alpha \propto \frac{n}{n-M}$ and the regularization parameter $\lambda$ is the unique positive number satisfying 
    \begin{align*}
        \text{Tr}\left[\mathbf{K}(\mathbf{K} + \lambda n I)^{-1}\right] = M/n.
    \end{align*}
\end{theorem}

The above bound estimates the approximation between the ridgeless RF estimator $\frrrf$ and the ridge regression $\fridgekrr$, such that the ridgeless RF essentially imposes an \textit{implicit regularization}.
For the shift-invariant kernels $K(\xx, \xx') = h(\|\xx - \xx'\|)$, i.e. Gaussian kernel $K(\xx, \xx') = \exp(-\|\xx-\xx'\|^2/(2\sigma^2))$, the trace of matrix is a constant $\text{Tr}(\mathbf{K}) = n$.
The number of random features needs $M = \Omega(n^{0.75})$ to achieve the convergence rate $\mO(1/\sqrt{n})$ for Theorem \ref{thm.rf_underparameterized}.
Note that $\widetilde{\mathcal{N}}(\lambda) := \text{Tr}\left[\mathbf{K}(\mathbf{K} + \lambda n I)^{-1}\right]$ is known as the empirical \textit{effective dimension}, which has been used to control the capacity of hypothesis space \cite{caponnetto2007optimal} and sample points via leverage scores \cite{rudi2018fast}.
Theorem \ref{thm.rf_underparameterized} states the \textit{implicit regularization} effect of random features, of which the regularizer parameter is related to the features dimensional $M$.

Together with Theorem \ref{thm.rf_overparameterized}, Theorem \ref{thm.rf_underparameterized} and Figure \ref{fig.exp1_stochastic} (a), we discuss the influence from different feature dimensions $M$ for ridgeless RF predictor $\frrrf$:

1) In the underparameterized regime $M < n$, the regularization parameter is inversely proportional to the feature dimensional $\lambda \propto 1/M$.
The effect of implicit regularity becomes greater as we reduce the features dimensional, while the implicit regularity reduces to zero as $M \to n$.
As the increase of $M$, the test error drops at first and then rises due to overfitting (or explored variance).

2)  At the interpolation threshold $M = n$, the condition $\text{Tr}\left[\mathbf{K}(\mathbf{K} + \lambda I)^{-1}\right] = M$ leads to $\lambda = 0$. Thus, $M = n$ not only split the overparameterized regime and the underparameterized regime, but also is the start of implicit regularization for ridgeless RF.
    At this point, the variance of ridgeless RF predictor explodes, leading to double descent curve.

3) In the overparameterized case $M > n$, the ridgeless RF predictor is an unbiased estimator ridgeless kernel predictor, but the effective ridge goes to zero.
As the number of random features increases, the test error of ridgeless RF drops again.

\begin{remark}[Excess risk bounds]
    From \eqref{eq.error_decomposition}, one can derive the entire excess risk bound for ridgeless RF-SGD $\|\frrsgd - \frho\|$ by using the existing ridge regression bound $\|\fridgekrr - \frho\|$ in \cite{Bartlett2005lrc,caponnetto2007optimal} for Theorem \ref{thm.rf_underparameterized}, and ridgeless regression bound $\|\fkrr - \frho\|$ in \cite{liang2020just} for Theorem \ref{thm.rf_overparameterized}.
    These risk bounds usually characterized the bias-variance decomposition with $\|f - \frho\|_\Ltwo^2 \leq {\boldsymbol B} + {\boldsymbol V}, ~ \forall f \in \Ltwo$, where the bias usually can be bounded by ${\boldsymbol B} \lesssim \frac{1}{n} \sqrt{\text{Tr}(\mathbf{K})}$.
    For the conventional kernel methods, the kernel matrix is fixed and thus $\text{Tr}(\mathbf{K})$ is a constant.

\end{remark}

\begin{algorithm}[t]
    \caption{\small Ridgeless RF with Tunable Kernels (\texttt{RFTK})}
    \label{alg.rftk}
    \begin{algorithmic}[1]
        \REQUIRE Training data $(\XX, \yy)$ and 
        feature mapping $\phi: \mathbb{R}^d \to \mathbb{R}^D.$
        Hyparameters $\sigma^2, \beta,T, b, \gamma, \eta$ and $s$.

        \ENSURE The ridgeless RF model $\ww_T$ and the learned $\OO$.
        \FOR{$t = 1, 2, \cdots, T$}
        \STATE Sample a batch examples $(\xx_i, ~ \yy_i)_{i=1}^b \in (\XX, \yy).$
        \STATE Update RF model weights $\ww_t = \ww_{t-1} - \gamma \frac{\partial \mL(\ww, \OO)}{\partial \ww}$
        \IF{$t \% s == 0 $}
        \STATE Optimize frequency matrix $\OO_t = \OO_{t-1} - \eta \frac{\partial \mL(\ww, \OO)}{\partial \OO}$
        \ENDIF
        \ENDFOR
    \end{algorithmic}
  \end{algorithm}  

\section{Random Features with Tunable Kernels}
By noting that $\text{Tr}(\mathbf{K})$ play a dominate role in random features error in Theorem \ref{thm.rf_underparameterized} and bias, we thus design a tunable kernel learning algorithm by reducing the trace $\text{Tr}(\mathbf{K})$.
We first transfer the trace to trainable form by random features
\begin{align*}
    \text{Tr}(\mathbf{K}) \approx \text{Tr} \left(\phi(\XX)\phi(\XX)^\top\right) = \|\phi(\XX)\|_F^2,
\end{align*}
where $\phi: \mX \to \RR^M$ depends on the spectral density $\pi$ according to \eqref{eq.rrrf} and $\|\cdot\|_F^2$ is the squared Frobenius norm for a matrix.
Since $\|\cdot\|_F^2$ is differentiable w.r.t. $\OO$, we thus can optimize the kernel density with backpropagation.
For example, considering the random Fourier features \cite{rahimi2007random}, the feature mappings can be written as 
\begin{align}
    \label{eq.rff}
    \phi(\xx) = \frac{1}{\sqrt{M}} \cos(\OO^\top\xx + \bb),
\end{align}
where the frequency matrix $\OO = [\oo_1, \cdots, \oo_M] \in \RR^{d \times M}$ composed $M$ vectors drawn i.i.d. from a Gaussian distribution $\mathcal{N}(0, \frac{1}{\sigma^2}\mathbf{I}) \in \RR^d$.
The phase vectors ${\boldsymbol b} = [b_1, \cdots, b_M] \in \RR^M$ are drawn uniformly from $[0, 2\pi]$.

\begin{figure*}[t]
    \centering
    \subfigure[The variance factor $\alpha$ w.r.t. $M/n$]{
		\centering
        \includegraphics[width=.28\textwidth]{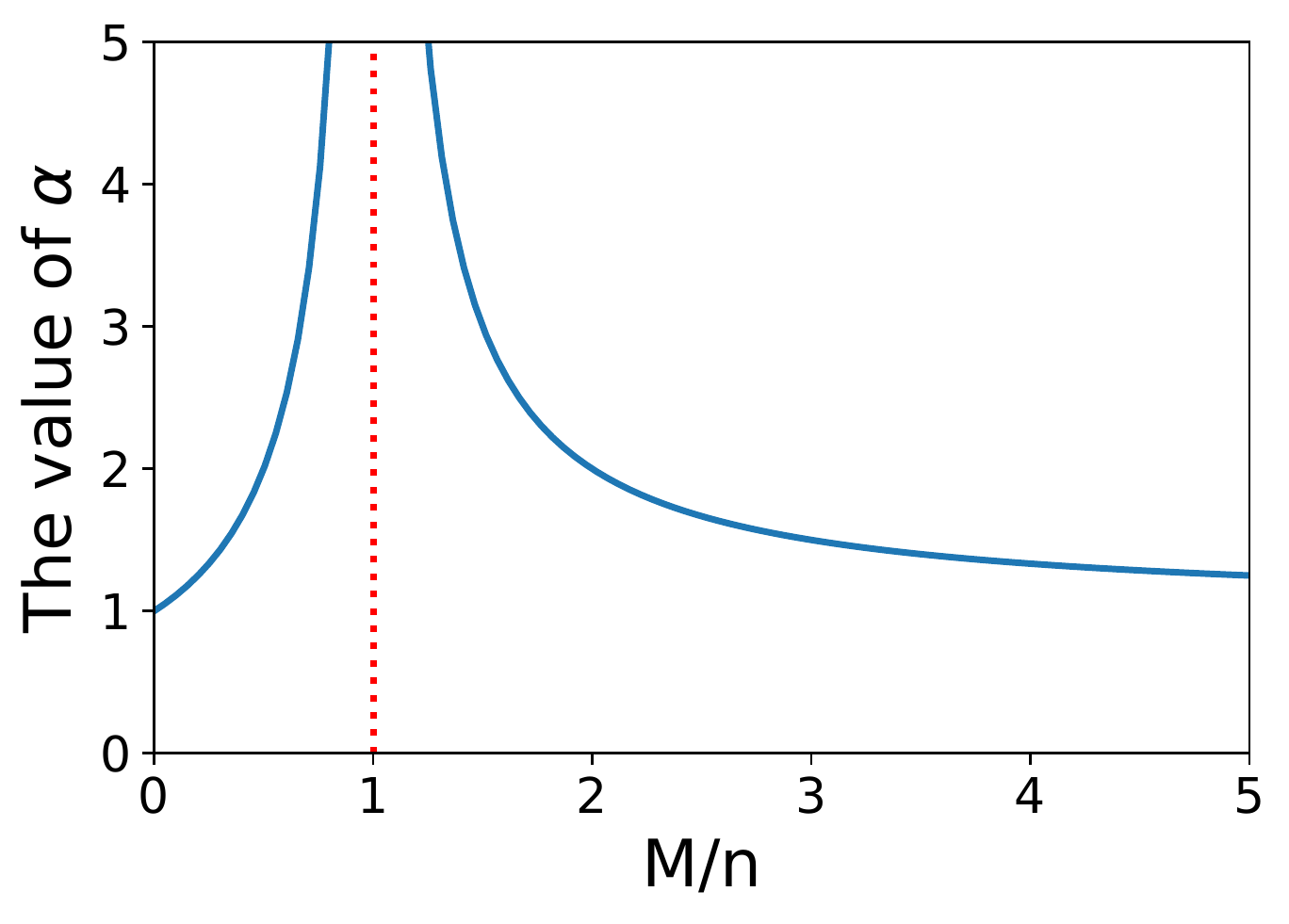}
	}
    \subfigure[Stochastic error vs. factors $b, \gamma$ and $t$]{
		\centering
        \includegraphics[width=.28\textwidth]{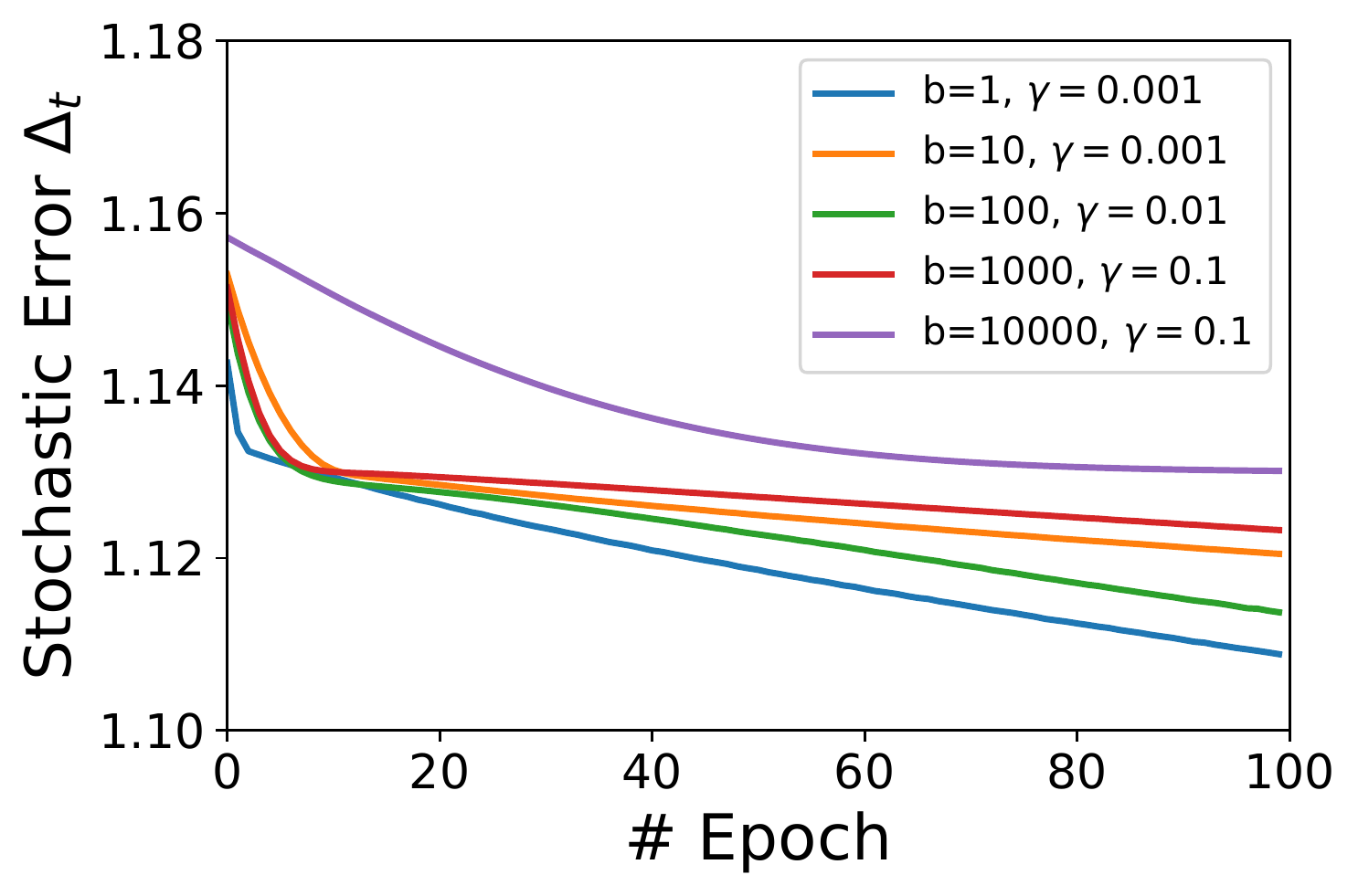}
	}
    \subfigure[MSE vs. the factors $b, \gamma$ and $t$]{
		\centering
        \includegraphics[width=.28\textwidth]{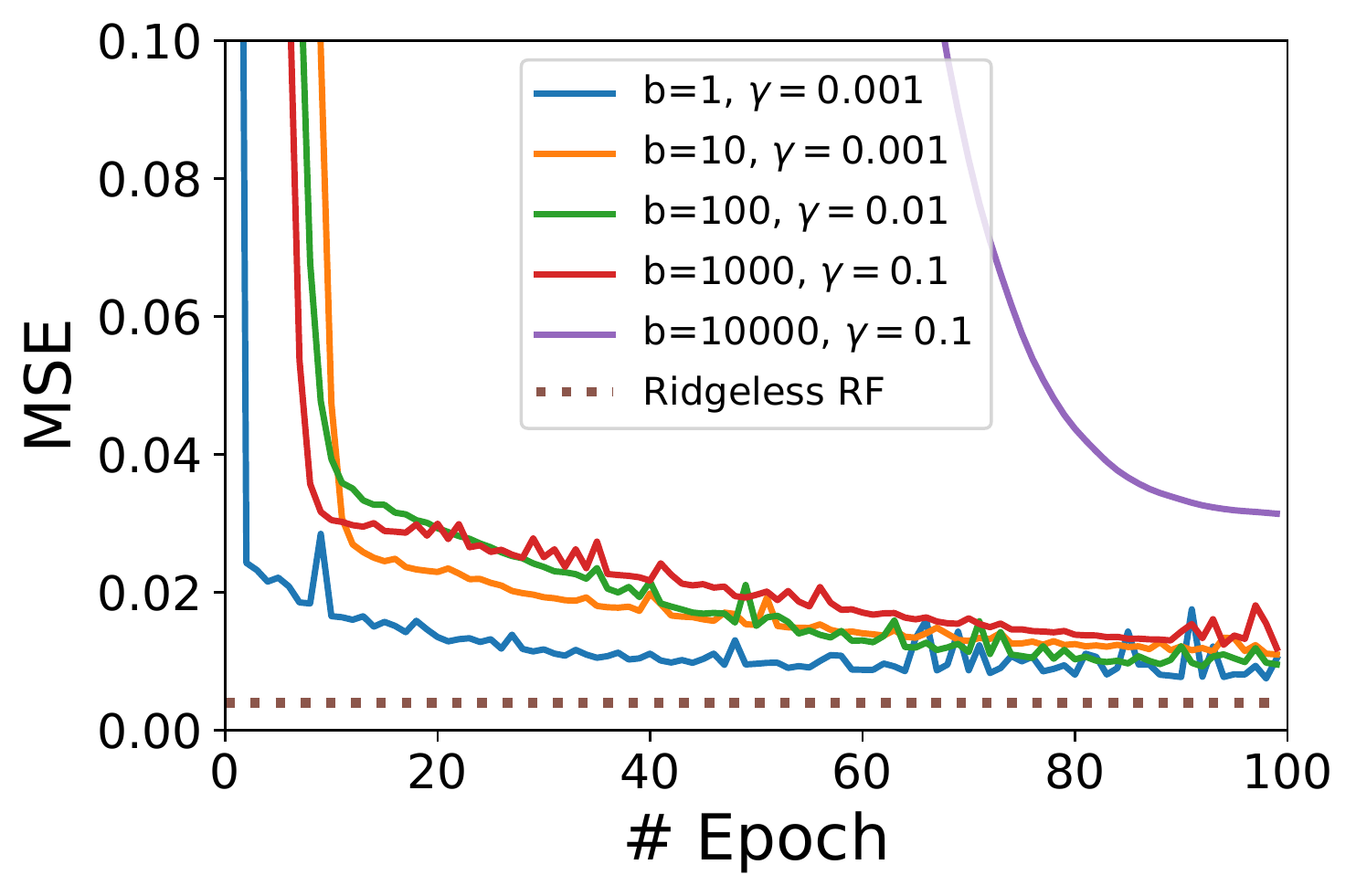}
	}
    \caption{(a) The variance factor $\alpha$ w.r.t. ratio $M/n$, which make the variance explores near the transition point $M = n$. (b) Empirical stochastic gradient errors $\Delta_t = \frac{1}{n} \sum_{i=1}^n|\frrsgd(\xx_i) - \frrrf(\xx_i)| $ of the ridgeless RF-SGD predictors on $n=10000$ synthetic examples. (c) The test MSE of all ridgeless RF-SGD estimators. }
    \label{fig.exp1_stochastic}
\end{figure*}

\begin{figure*}[t]
    \centering
    \subfigure[Test errors v.s. regulerizations]{
		\centering
        \includegraphics[width=.28\textwidth]{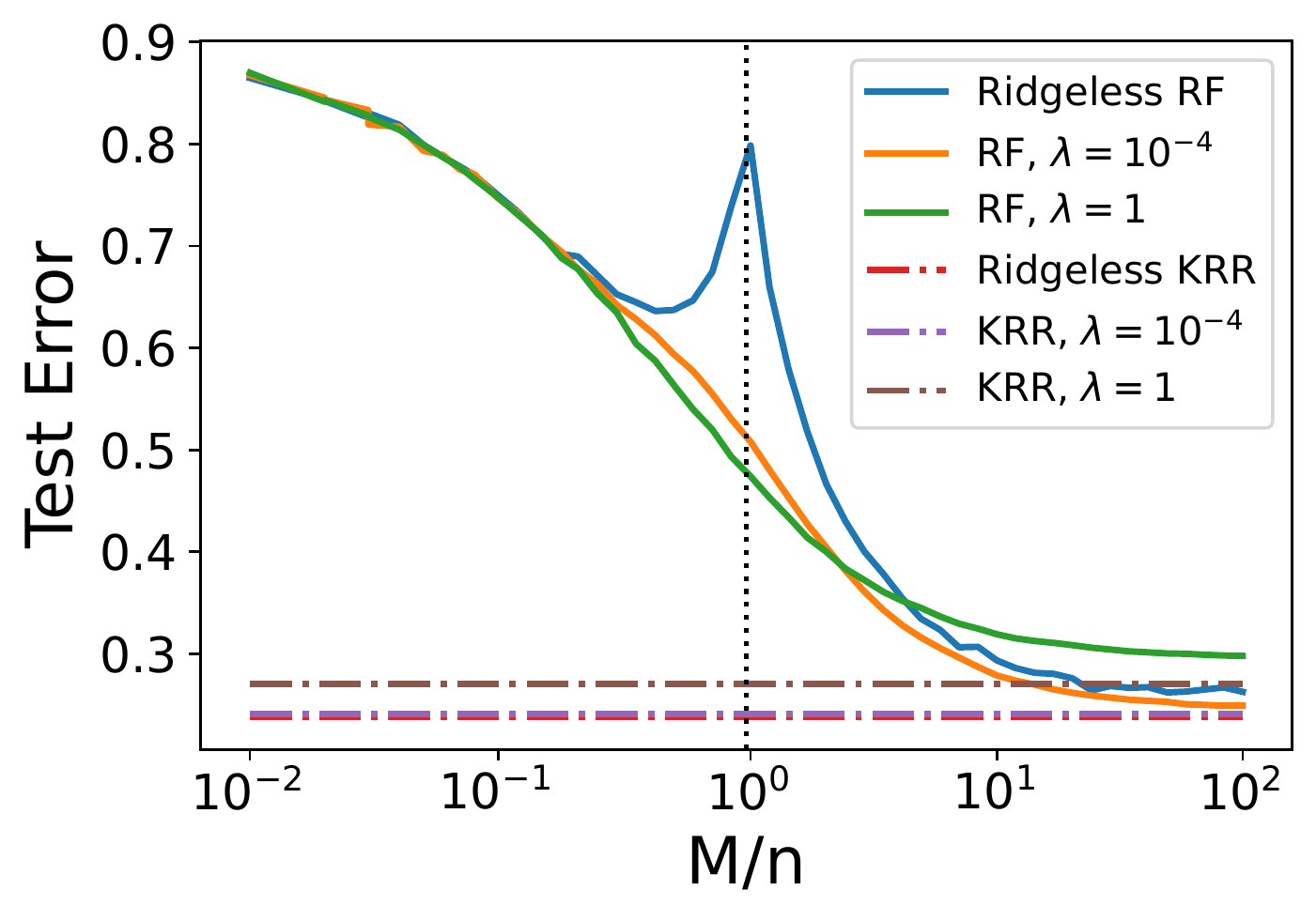}
	}
    \subfigure[Test accuracies of compared methods]{
		\centering
        \includegraphics[width=.28\textwidth]{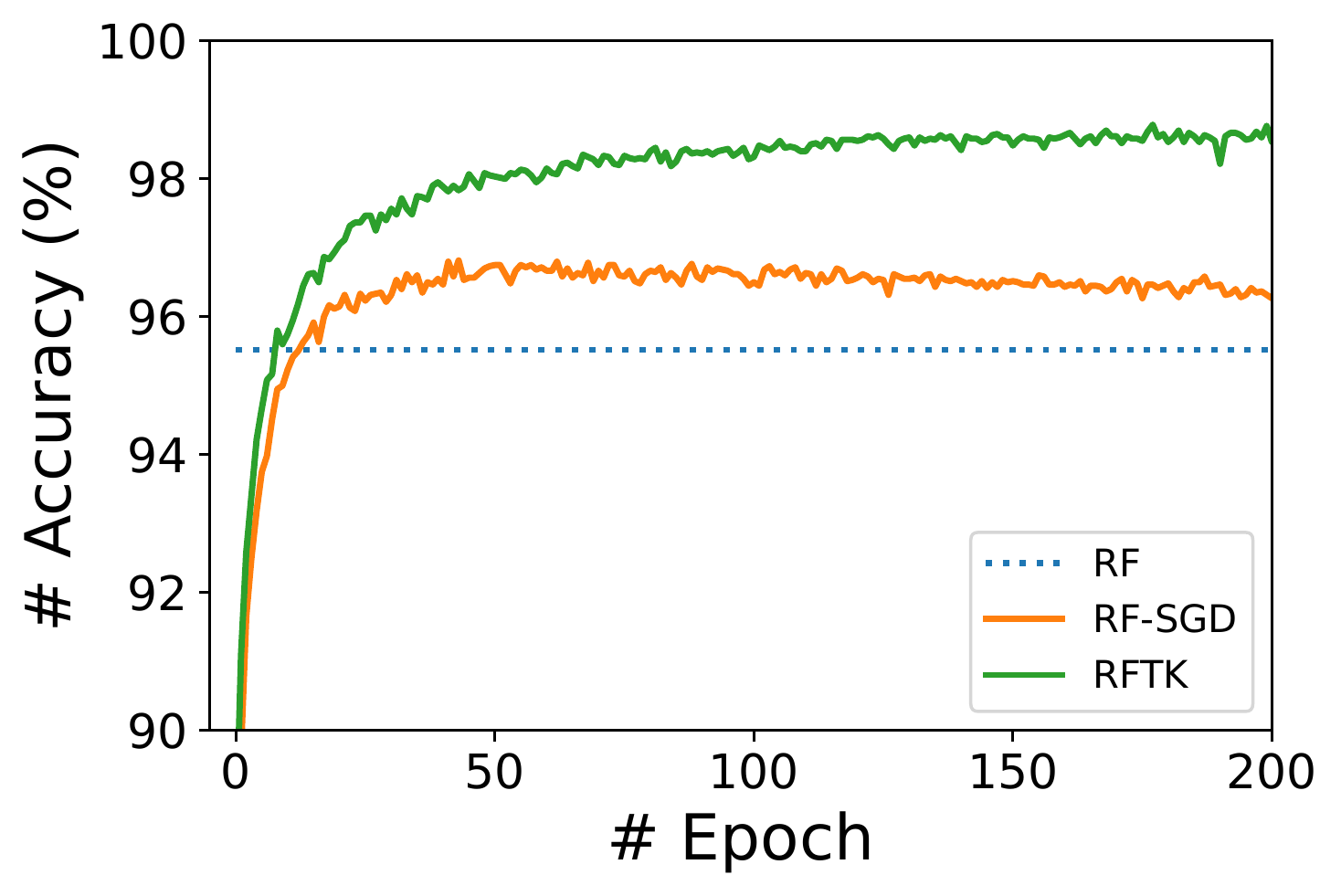}
	}
    \subfigure[Training loss and the kernel trace]{
		\centering
        \includegraphics[width=.28\textwidth]{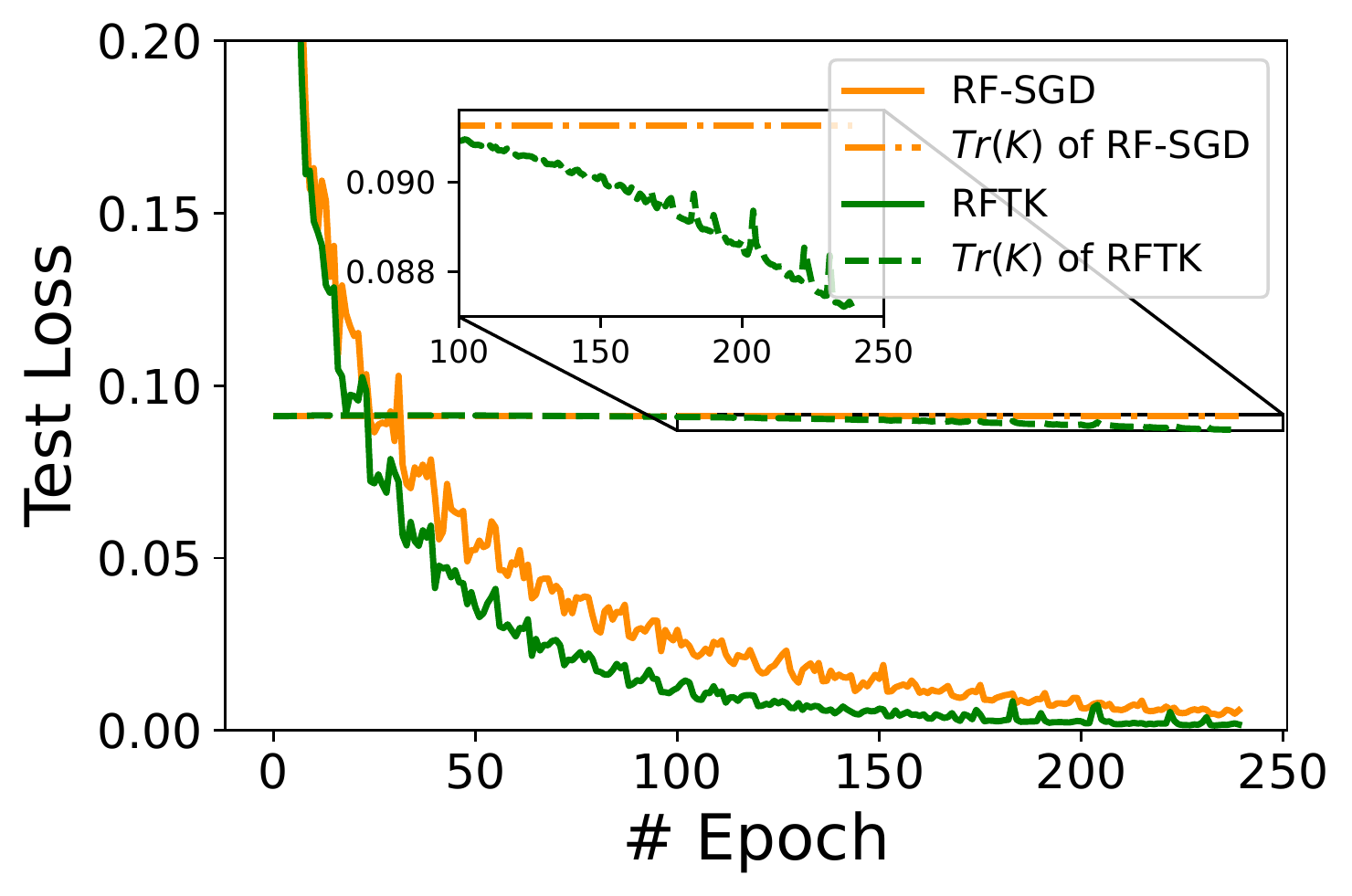}
	}
    \caption{(a) Test errors of the RF predictors (solid lines) and kernel predictors (dashed lines) w.r.t. different regularization. Note that, the ridgeless RF predictors exhibit a double descent curve.  (b) Test accuracies of the compared methods on the MNIST dataset. (c) Training loss (solid lines) and the trace of kernel $\text{Tr}(\mathbf{K})$ (dashed lines) of the RF predictors on the MNIST dataset.}
    \label{fig.exp2_rf_exp3}
\end{figure*}

\begin{table*}[t]
    \centering
    \begin{tabular}{@{\extracolsep{0.6cm}}l|cc|cc|c}
        \toprule
                   & Kernel Ridge             & Kernel Ridgeless            & RF                        & RF-SGD                       & \texttt{RFTK}             \\ \hline
                   dna  &52.83$\pm$1.66   &49.67$\pm$18.54  &52.83$\pm$1.66  &51.33$\pm$1.89  &\textbf{92.92$\pm$0.89} \\
                   letter  &\textbf{96.54$\pm$0.25}   &\underline{96.40$\pm$0.15}  &95.33$\pm$0.32  &91.74$\pm$0.46  &\underline{96.17$\pm$0.29} \\
                   pendigits  &97.46$\pm$0.42   &90.67$\pm$4.75  &96.91$\pm$0.43  &46.04$\pm$5.62  &\textbf{98.64$\pm$0.43} \\
                   segment  &82.99$\pm$1.85   &56.71$\pm$10.71  &83.44$\pm$1.69  &37.75$\pm$9.11  &\textbf{94.55$\pm$1.52} \\
                   satimage  &90.43$\pm$0.48   &88.79$\pm$0.77  &87.67$\pm$0.89  &90.33$\pm$1.36  &\textbf{90.79$\pm$1.23} \\
                   usps  &92.49$\pm$0.70   &87.47$\pm$7.38  &94.38$\pm$0.60  &49.81$\pm$3.52  &\textbf{97.29$\pm$0.61} \\
                   svmguide2  &81.90$\pm$2.78   &70.13$\pm$4.91  &66.20$\pm$4.64  &81.65$\pm$4.25  &\textbf{82.78$\pm$4.84} \\
                   vehicle  &63.00$\pm$2.80   &79.35$\pm$2.89  &75.94$\pm$2.88  &74.24$\pm$3.92  &\textbf{80.06$\pm$4.49} \\
                   wine  &39.17$\pm$6.27   &48.89$\pm$13.68  &91.11$\pm$19.40  &43.89$\pm$6.19  &\textbf{98.33$\pm$1.36} \\
                   shuttle  &/   &/  &79.08$\pm$26.76  &98.94$\pm$0.48  &\textbf{99.63$\pm$0.19} \\
                   Sensorless  &/   &/  &32.92$\pm$8.22  &17.72$\pm$3.84  &\textbf{86.10$\pm$0.73} \\
                   MNIST  &/   &/  &95.52$\pm$0.16  &96.61$\pm$0.11  &\textbf{98.09$\pm$0.07} \\
        \bottomrule
    \end{tabular}
    \caption{
        \normalsize  Classification accuracy (\%) for classification datasets. We bold the results with the best method and underline the ones that are not significantly worse than the best one.}
    \label{tabel.accuracy}
\end{table*}

Theoretical findings illustrate that smaller trace of kernel matrix $\text{Tr}(\mathbf{K})$ can lead to better performance.
We first propose a bi-level optimization learning problem
\begin{equation}
    \begin{aligned}
        \label{obj.bi_level}
        \min_{\ww} \quad & \frac{1}{n}\sum_{i=1}^n \left(\ww^\top \phi(\xx_i) - y_i \right)^2 \\
        \text{s.t.} \quad &\OO^* = \argmin_{\OO} ~ \|\phi(\XX)\|_F^2.
    \end{aligned}
\end{equation}

The above objective includes two steps: 1) given a certain kernel, i.e. the frequency matrix $\OO = [\oo_1, \cdots, \oo_M] \sim \pi(\ww)$, the algorithm train the ridgeless RF model $\ww$; 2) given a trained RF model $\ww$, the algorithm optimize the spectral density (the kernel) by updating the frequency matrix $\OO$.

To accelerate the solve of \eqref{obj.bi_level}, we optimize $\ww$ and $\OO$ jointly by minimize the the following objective
\begin{equation}
    \label{equation.primal-objective}
        \mL(\ww, \OO) =  \frac{1}{n} ~ \sum_{i=1}^n \left(f(\xx_i) - \yy_i\right)^2 + \beta \|\phi(\boldsymbol{X})\|_F^2,
\end{equation}
where $\beta$ is a hyperparameter to balance the effect between empirical loss and the trace of kernel matrix.
Here, the update of $\ww$ is still only dependent on the squared loss (thus it is still ridgeless), but the update of $\OO$ is related to both the squared loss and Frobenius norm.


Therefore, the time complexity for update $\ww$ once is $\mO(Mb)$. 
However, since the trace is defined on all data, the update of $\OO$ is relevant to all data and time complexity is $\mO(nMd)$, which is infeasible to update in every iterations.
Meanwhile, the kernel needn't to update frequently, and thus we apply an asynchronous strategy for optimizing the spectral density.
As shown in Algorithm \ref{alg.rftk}, the frequency matrix $\OO$ is updated after every $s$ iterations for the update of $\ww$.
The total complexity for the asynchronous strategy is $\mO(nMdT/s)$.

\section{Experiments}

We use random Fourier feature defined in \eqref{eq.rff} to approximate the Gaussian kernel $K(\xx, \xx') = \exp(- \sigma^2 \|\xx-\xx'\|^2 /2)$.
Note that, the corresponding random Fourier features \eqref{eq.rff} is with the frequency matrix $\OO \sim \mathcal{N}(0, \sigma^2)$.
We implement all code on Pytorch \footnote{\scriptsize\url{https://github.com/superlj666/Ridgeless-Regression
-with-Random-Features}} and tune hyperparameters over $\sigma^2 \in \{0.01, \cdots, 1000\}$, $\lambda = \{0.1, \cdots, 10^{-5}\}$ by grid search.

\subsection{Numerical Validations}

\subsubsection{Factors of Stochastic Gradient Methods}
We start with a nonlinear problem $y = \min(-\ww^\top \xx, \ww^\top \xx) + \epsilon$, where $\epsilon \sim \mathcal{N}(0, 0.2)$ is the label noise, and $\xx \sim \mathcal{N}(0, \mathbf{I})$.
Setting $d=10$, we generate $n=10000$ samples for training and $2500$ samples for testing.
The optimal kernel hyperparameter is $\sigma^2 = 100$ after grid search.

As stated in Theorem \ref{thm.sgd}, the stochastic gradients error is determined by the batch size $b$, learning rate $\gamma$ and the iterations $t$.
To explore effects of these factors, we evaluate the approximation between the ridgeless RF-SGD estimator $\frrsgd$ and the ridgeless RF estimator $\frrrf$ on the simulated data.
Given a batch size $b$, we tune the learning rate $\gamma$ w.r.t. the MSE over $\gamma \in \{10^1, 10^0, \cdots, 10^{-4}\}$.
As shown in Figure \ref{fig.exp1_stochastic} (b), (c), give the $b, \gamma$, we estimate the ideal iterations $t = \text{\# Epoch} * 10000 / b$ after which the error drops slowly
As the step size $b$ increases, the learning rates $\gamma$ become larger while the iterations reduces.
This coincides with the tradeoffs among these factors $b/\gamma + 1/(\gamma t)$ in Theorem \ref{thm.sgd}, where the balance between $b/\gamma$ and $1/\gamma t$ leads to better performance.

Comparing Figure \ref{fig.exp1_stochastic} (b) and (c), we find that: 1) After the same passes over the data (same epoch), the stochastic gradient error and MSE of the algorithms with smaller batch sizes are smaller with faster drop speeds. 2) The stochastic error still decreases after the MSE converges, where the other error terms dominates the excess risk, i.e. bias of predictors.

\subsubsection{Double Descent Curve in Random Features}
To characterize different behaviors of the random features error w.r.t. $M$, we fixed the number of training examples $n$ and changes the random features dimensional $M$.
Figure \ref{fig.exp2_rf_exp3} (a) reports test errors in terms of different ratios $M/n$, illustrating that: 1) Kernel methods with smaller regularization lead to smaller test errors, and the reason may be regularization hurts the performance when the training data is "clean". 2) Test errors of RF predictors converges to the corresponding kernel predictors, because a larger $M$ have better approximation to the kernel. This coincides with Theorem \ref{thm.rf_overparameterized} and \ref{thm.rf_underparameterized} that a larger $M$ reduces the random features error. 3) When the regularization term $\lambda$ is small, the predictors leads to \textit{double descent curves} and $M = n$ is the transition point dividing the underparamterized and overparameterized regions.
The smaller regularity, the more obvious curve the RF predictor exhibits.

In the overparameterized case, the error of RF predictor converges to corresponding kernel predictor as $M$ increases, verifying the results in Theorem \ref{thm.rf_overparameterized} that the variance term reduces given more features when $M > n$.
In the underparamterized regime, the errors of RF predictors drop first and then increase, and finally explodes at $M=n$.
These empirical results validates theoretical findings in Theorem \ref{thm.rf_underparameterized} that the variance term dominates the random features near $M = n$ and it is significantly large as shown in Figure \ref{fig.exp1_stochastic} (a).

\begin{table}[t]
    \centering
    \resizebox{\columnwidth}{!}{
    \setlength\extrarowheight{3pt}
    \begin{tabular}{lllm{1.7cm}}
        \toprule
        Methods     & Time         & Density & Regularizer \\ \hline
        Kernel Ridge            & $\mO(n^3)$     & Assigned & Ridge        \\
        Kernel Ridgeless            & $\mO(n^3)$     & Assigned & Ridgeless        \\
        RF           & $\mO(nM^2 + M^3)$ & Assigned & Ridgeless       \\
        RF-SGD           & $\mO(MbT)$     & Assigned  & Ridgeless        \\
        \texttt{RFTK}          & $\mO(nMbT/s)$ & Learned  & $\|\phi(\XX)\|_F^2$        \\ 
        \bottomrule
    \end{tabular}
    }
    \caption{Compared algorithms.}
    \label{tab.compared-methods}
\end{table}

\subsubsection{Benefits from Tunable Kernel}
Motivated by the theoretical findings that the trace of kernel matrix $\text{Tr}(\mathbf{K})$ influences the performance, we design a tunable kernel method \texttt{RFTK} that adaptively optimizes the spectral density in the training.
To explore the influence of factors upon convergence, we evaluate both test accuracy and training loss on the MNIST dataset.
Compared with the exact random features (RF) and random features with stochastic gradients (RF-SGD), we conduct experiments on the entire MNIST datasets.
From Figure \ref{fig.exp2_rf_exp3} (b), we find there is a significant accuracy gap between RF-SGD and \texttt{RFTK}.
Figure \ref{fig.exp2_rf_exp3} (c) indicates the trace of kernel matrix term takes affects after the current model fully trained near $100$ epochs, and it decrease fast.
Specifically, since the kernel is continuously optimized, more iterations can still improve the its generalization performance, while the accuracy of RF-SGD decreases after $100$ epochs because of the overfitting to the given hypothesis.

Figure \ref{fig.exp2_rf_exp3} (b) and (c) explain the benefits from the penalty of $\text{Tr}(\mathbf{K})$ that optimizes the kernel during the training, avoiding kernel selection.
Empirical studies shows that smaller $\text{Tr}(\mathbf{K})$ guarantees better performance, coinciding with the theoretical results in Theorem \ref{thm.sgd} and Theorem \ref{thm.rf_underparameterized} that both stochastic gradients error and random features error depends on $\text{Tr}(\mathbf{K})$. 

\subsection{Empirical Results}
Compared with related algorithms listed in Table \ref{tab.compared-methods}, we evaluate the empirical behavior of our proposed algorithm \texttt{RFTK} on several benchmark datasets.
Only the kernel ridge regression makes uses of the ridge penalty $\|\ww\|_2^2$, while the others are ridgeless.
For the sake of presentation, we perform regression algorithms on classification datasets with one-hot encoding and cross-entropy loss.
We set $b=32$ and $100$ epochs for the training, and thus the stop iteration is $T = 100n/32$.
Before the training, we tune the hyperparameters $\sigma^2, \lambda, \gamma, \beta$ via grid search for algorithms on each dataset.
To obtain stable results, we run methods on each dataset 10 times with randomly partition such that 80\% data for training and $20\%$ data for testing.
Further, those multiple test errors allow the estimation of the statistical significance among methods.

Table \ref{tabel.accuracy} reports the test accuracies for compared methods over classification tasks.
It demonstrates that: 1) In some cases, the proposed \texttt{RFTK} remarkably outperforms the compared methods, for example dna, segment and Sensorless. That means \texttt{RFTK} can optimize the kernel even with a bad initialization, which makes it more flexible than the schema (kernel selection + learning).
2) For many tasks, the accuracies of \texttt{RFTK} are also significantly higher than kernel predictors, i.e. dna, because the learned kernel is more suitable for these tasks.
3) Compared to kernel predictors, \texttt{RFTK} leads to similar or worse results in some cases. The reason is that kernel hyperparameter $\sigma^2$ has been tuned and in these cases they are near the optimal one, and thus the spectral density is changed a little by \texttt{RFTK}.

\section{Conclusion}
We study the generalization properties for ridgeless regression with random features, and devise a kernel learning algorithm that asynchronously tune spectral kernel density during the training.
Our work filled the gap between the generalization theory and practical algorithms for ridgeless regression with random features.
The techniques presented here provides theoretical and algorithmic insights for understanding of neural networks and designing new algorithms. 

\section*{Acknowledgments}
This work was supported in part by the Excellent Talents Program of Institute of Information Engineering, CAS, the Special Research Assistant Project of CAS, the Beijing Outstanding Young Scientist Program (No. BJJWZYJH012019100020098), Beijing Natural Science Foundation (No. 4222029), and National Natural Science Foundation of China (No. 62076234, No. 62106257).

\bibliographystyle{named}
\bibliography{all}

\begin{thebibliography}{}

\bibitem[\protect\citeauthoryear{Advani \bgroup \em et al.\egroup
  }{2020}]{advani2020high}
Madhu~S Advani, Andrew~M Saxe, and Haim Sompolinsky.
\newblock High-dimensional dynamics of generalization error in neural networks.
\newblock {\em Neural Networks}, 132:428--446, 2020.

\bibitem[\protect\citeauthoryear{Bartlett \bgroup \em et al.\egroup
  }{2005}]{Bartlett2005lrc}
Peter~L. Bartlett, Olivier Bousquet, and Shahar Mendelson.
\newblock Local {R}ademacher complexities.
\newblock {\em The Annals of Statistics}, 33(4):1497--1537, 2005.

\bibitem[\protect\citeauthoryear{Bartlett \bgroup \em et al.\egroup
  }{2020}]{bartlett2020benign}
Peter~L Bartlett, Philip~M Long, G{\'a}bor Lugosi, and Alexander Tsigler.
\newblock Benign overfitting in linear regression.
\newblock {\em Proceedings of the National Academy of Sciences},
  117(48):30063--30070, 2020.

\bibitem[\protect\citeauthoryear{Belkin \bgroup \em et al.\egroup
  }{2018}]{belkin2018understand}
Mikhail Belkin, Siyuan Ma, and Soumik Mandal.
\newblock To understand deep learning we need to understand kernel learning.
\newblock {\em arXiv preprint arXiv:1802.01396}, 2018.

\bibitem[\protect\citeauthoryear{Caponnetto and
  De~Vito}{2007}]{caponnetto2007optimal}
Andrea Caponnetto and Ernesto De~Vito.
\newblock Optimal rates for the regularized least-squares algorithm.
\newblock {\em Foundations of Computational Mathematics}, 7(3):331--368, 2007.

\bibitem[\protect\citeauthoryear{Carratino \bgroup \em et al.\egroup
  }{2018}]{carratino2018learning}
Luigi Carratino, Alessandro Rudi, and Lorenzo Rosasco.
\newblock Learning with sgd and random features.
\newblock In {\em Advances in Neural Information Processing Systems 31
  (NeurIPS)}, pages 10192--10203, 2018.

\bibitem[\protect\citeauthoryear{Geiger \bgroup \em et al.\egroup
  }{2020}]{geiger2020scaling}
Mario Geiger, Arthur Jacot, Stefano Spigler, Franck Gabriel, Levent Sagun,
  St{\'e}phane d’Ascoli, Giulio Biroli, Cl{\'e}ment Hongler, and Matthieu
  Wyart.
\newblock Scaling description of generalization with number of parameters in
  deep learning.
\newblock {\em Journal of Statistical Mechanics: Theory and Experiment},
  2020(2):023401, 2020.

\bibitem[\protect\citeauthoryear{Ghorbani \bgroup \em et al.\egroup
  }{2021}]{ghorbani2021linearized}
Behrooz Ghorbani, Song Mei, Theodor Misiakiewicz, and Andrea Montanari.
\newblock Linearized two-layers neural networks in high dimension.
\newblock {\em The Annals of Statistics}, 49(2):1029--1054, 2021.

\bibitem[\protect\citeauthoryear{Hastie \bgroup \em et al.\egroup
  }{2019}]{hastie2019surprises}
Trevor Hastie, Andrea Montanari, Saharon Rosset, and Ryan~J Tibshirani.
\newblock Surprises in high-dimensional ridgeless least squares interpolation.
\newblock {\em arXiv preprint arXiv:1903.08560}, 2019.

\bibitem[\protect\citeauthoryear{Jacot \bgroup \em et al.\egroup
  }{2020}]{jacot2020implicit}
Arthur Jacot, Berfin Simsek, Francesco Spadaro, Cl{\'e}ment Hongler, and Franck
  Gabriel.
\newblock Implicit regularization of random feature models.
\newblock In {\em International Conference on Machine Learning}, pages
  4631--4640. PMLR, 2020.

\bibitem[\protect\citeauthoryear{Li and Liu}{2022}]{li2022optimal}
Jian Li and Yong Liu.
\newblock Optimal rates for distributed learning with random features.
\newblock In {\em Proceedings of the 31st International Joint Conference on
  Artificial Intelligence (IJCAI)}, 2022.

\bibitem[\protect\citeauthoryear{Li \bgroup \em et al.\egroup
  }{2018}]{li2018multi}
Jian Li, Yong Liu, Rong Yin, Hua Zhang, Lizhong Ding, and Weiping Wang.
\newblock Multi-class learning: From theory to algorithm.
\newblock In {\em Advances in Neural Information Processing Systems 31}, pages
  1591--1600, 2018.

\bibitem[\protect\citeauthoryear{Li \bgroup \em et al.\egroup
  }{2019}]{li2019multi}
Jian Li, Yong Liu, Rong Yin, and Weiping Wang.
\newblock Multi-class learning using unlabeled samples : Theory and algorithm.
\newblock In {\em Proceedings of the 28th International Joint Conference on
  Artificial Intelligence (IJCAI)}, 2019.

\bibitem[\protect\citeauthoryear{Li \bgroup \em et al.\egroup
  }{2020}]{li2020automated}
Jian Li, Yong Liu, and Weiping Wang.
\newblock Automated spectral kernel learning.
\newblock In {\em Thirty-Four AAAI Conference on Artificial Intelligence},
  2020.

\bibitem[\protect\citeauthoryear{Liang and Rakhlin}{2020}]{liang2020just}
Tengyuan Liang and Alexander Rakhlin.
\newblock Just interpolate: Kernel “ridgeless” regression can generalize.
\newblock {\em The Annals of Statistics}, 48(3):1329--1347, 2020.

\bibitem[\protect\citeauthoryear{Mei and
  Montanari}{2019}]{mei2019generalization}
Song Mei and Andrea Montanari.
\newblock The generalization error of random features regression: Precise
  asymptotics and the double descent curve.
\newblock {\em Communications on Pure and Applied Mathematics}, 2019.

\bibitem[\protect\citeauthoryear{Nakkiran \bgroup \em et al.\egroup
  }{2021}]{nakkiran2021deep}
Preetum Nakkiran, Gal Kaplun, Yamini Bansal, Tristan Yang, Boaz Barak, and Ilya
  Sutskever.
\newblock Deep double descent: Where bigger models and more data hurt.
\newblock {\em Journal of Statistical Mechanics: Theory and Experiment},
  2021(12):124003, 2021.

\bibitem[\protect\citeauthoryear{Rahimi and Recht}{2007}]{rahimi2007random}
Ali Rahimi and Benjamin Recht.
\newblock Random features for large-scale kernel machines.
\newblock In {\em Advances in Neural Information Processing Systems 21 (NIPS)},
  pages 1177--1184, 2007.

\bibitem[\protect\citeauthoryear{Rahimi and Recht}{2008}]{rahimi2008weighted}
Ali Rahimi and Benjamin Recht.
\newblock Weighted sums of random kitchen sinks: Replacing minimization with
  randomization in learning.
\newblock In {\em Advances in Neural Information Processing Systems 22 (NIPS)},
  pages 1313--1320, 2008.

\bibitem[\protect\citeauthoryear{Rudi and
  Rosasco}{2017}]{rudi2017generalization}
Alessandro Rudi and Lorenzo Rosasco.
\newblock Generalization properties of learning with random features.
\newblock In {\em Advances in Neural Information Processing Systems 30 (NIPS)},
  pages 3215--3225, 2017.

\bibitem[\protect\citeauthoryear{Rudi \bgroup \em et al.\egroup
  }{2018}]{rudi2018fast}
Alessandro Rudi, Daniele Calandriello, Luigi Carratino, and Lorenzo Rosasco.
\newblock On fast leverage score sampling and optimal learning.
\newblock In {\em Advances in Neural Information Processing Systems}, pages
  5672--5682, 2018.

\bibitem[\protect\citeauthoryear{Sch{\"o}lkopf \bgroup \em et al.\egroup
  }{1998}]{scholkopf1998nonlinear}
Bernhard Sch{\"o}lkopf, Alexander Smola, and Klaus-Robert M{\"u}ller.
\newblock Nonlinear component analysis as a kernel eigenvalue problem.
\newblock {\em Neural computation}, 10(5):1299--1319, 1998.

\bibitem[\protect\citeauthoryear{Smale and Zhou}{2007}]{smale2007learning}
Steve Smale and Ding-Xuan Zhou.
\newblock Learning theory estimates via integral operators and their
  approximations.
\newblock {\em Constructive approximation}, 26(2):153--172, 2007.

\bibitem[\protect\citeauthoryear{Zhang \bgroup \em et al.\egroup
  }{2021}]{zhang2021understanding}
Chiyuan Zhang, Samy Bengio, Moritz Hardt, Benjamin Recht, and Oriol Vinyals.
\newblock Understanding deep learning (still) requires rethinking
  generalization.
\newblock {\em Communications of the ACM}, 64(3):107--115, 2021.

\end{thebibliography}

\newpage
\appendix
\onecolumn
\section{Proofs}
In this section, we provide brief proofs for Theorem \ref{thm.sgd} and Theorem \ref{thm.rf_underparameterized}, respectively.
\begin{proof}[Proof for Theorem \ref{thm.sgd}]
    We introduce gradient descent algorithm for random features $\frrgd(\xx) = \langle \vv_t, \phi(\xx) \rangle, ~ \text{with}$ with $\vv_1 = 0$ and
    \begin{align}
        \label{}
        \vv_{t+1} = \vv_t - \frac{\gamma_t}{n} \sum_{i=1}^{n} \left(\langle \vv_t, \phi(\xx_i) \rangle - y_i\right) \phi(\xx_i).
    \end{align}

    By \eqref{eq.error_decomposition}, we consider the decomposition
    \begin{equation}
        \begin{aligned}
            \label{eq.rf_rfsgd.decomposition}
            \|\frrsgd - \frrrf\| \leq \|\frrsgd - \frrgd\| 
            + \|\frrgd - \frrrf\|.
        \end{aligned}
    \end{equation}

    From Lemma 5 in \cite{carratino2018learning}, under Assumptions \ref{asm.rf}, \ref{asm.moment}, let $\delta \in (0, 1)$, $\gamma \leq \frac{n}{9T\log\frac{n}{\delta}} \wedge \frac{1}{8(1 + \log T)}$ and $n \geq 32 \log^2 \frac{2}{\delta}$, the following bounded holds with the probability $1-2\delta$
    \begin{align}
        \label{eq.rfsgd_rfgd}
        \|\frrsgd - \frrgd\|_\Ltwo^2 \leq 208Bp \left(\frac{\gamma}{b}\right).
    \end{align}

    Setting $\gamma_t = \gamma$, one can derive the closed-form solution for stochastic gradient
    \begin{align*}
        \frrgd(\xx) 
        = \phi(\xx) \sum_{l=0}^{t-1} \left[I - \frac{\gamma}{n} \phi(\XX)^\top\phi(\XX) \right]^l \frac{\gamma}{n} \phi(\XX)^\top \yy 
    \end{align*}
    By the identity $\sum_{\tau = 0}^{t-1} (1-Z)^\tau Z = I - (I - Z)^t$, we get that
    \begin{align*}
        \frrgd(\xx)
        &=\phi(\xx) \sum_{l=0}^{t-1} \left[I - \frac{\gamma}{n} \phi(\XX)^\top\phi(\XX) \right]^l \left(\frac{\gamma}{n} \phi(\XX)^\top\phi(\XX)\right) \\
        &\quad \left(\frac{\gamma}{n} \phi(\XX)^\top\phi(\XX)\right)^{-1} \frac{\gamma}{n} \phi(\XX)^\top \yy \\
        &=\phi(\xx) \left[I - \left(I - \frac{\gamma}{n} \phi(\XX)^\top\phi(\XX)\right)^t\right] \\
        & \quad \left(\phi(\XX)^\top\phi(\XX)\right)^{-1} \phi(\XX)^\top \yy.
    \end{align*}

Denote the kernel operator $\mathbf{K}_M = \phi(\XX) \phi(\XX)^\top$.
Using the identity $(Z^\top Z)^{-1}Z = (Z^\top Z)^{-1}Z (Z Z^\top) (Z Z^\top)^{-1} = Z^\top (Z Z^\top)^{-1}$ for any matrix $Z$, we have
\begin{equation}
    \begin{aligned}
        \label{eq.proof.gd_rf}
        &|\frrgd(\xx) - \frrrf(\xx)| \\
        =&\Bigg| \phi(\xx) \left(I - \frac{\gamma}{n} \phi(\XX)^\top\phi(\XX)\right)^t \left(\phi(\XX)^\top\phi(\XX)\right)^{-1} \phi(\XX)^\top \yy \Bigg| \\
        =&\Bigg| \phi(\xx) \left(I - \frac{\gamma}{n} \phi(\XX)^\top\phi(\XX)\right)^t \phi(\XX)^\top\left(\phi(\XX) \phi(\XX)^\top \right)^{-1}\yy \Bigg| \\
        =&\Bigg| \phi(\xx) \phi(\XX)^\top \left(I - \frac{\gamma}{n} \mathbf{K}_M \right)^t \left(\mathbf{K}_M \right)^{-1}\yy \Bigg| \\
    \end{aligned}
\end{equation}

Consider the eigenvalue decomposition for kernel matrix
\begin{align*}
    \frac{1}{n}\mathbf{K}_M = U \Lambda U^\top,
\end{align*}
where $UU^\top = U^\top U = I$ records the eigenvectors and $\Lambda$ is a $n \times n$ diagonal with eigenvalues.
Consequently,
\begin{align*}
    \left(I - \frac{\gamma}{n} \mathbf{K}_M \right)^t = \left[U^{-1}U - \frac{\gamma}{n} U^{-1} \Lambda U\right]^t = U^{-1} \left[I - \frac{\gamma}{n} \Lambda \right]^t U.
\end{align*}
Then, it holds
\begin{equation}
    \begin{aligned}
        \label{eq.proof.iteration}
        &\left\|\phi(\XX)^\top \left(I - \gamma \mathbf{K}_M \right)^t\right\|^2_2 \\
        \leq &\left\|\phi(\XX)^\top \left(I - \gamma \mathbf{K}_M \right)^t\right\|^2_2 \\
        \leq &\left\|\phi(\XX)^\top U^{-1} \left[I - \gamma \Lambda \right]^t\right\|_2^2 \big\|U\big\|_2^2 \\
        \leq &\left\|\left[I - \gamma \Lambda \right]^t U^{-1} K(U^{-1}) \left[I - \gamma \Lambda \right]^t\right\|_2 \\
        = &\left\|\left[I - \gamma \Lambda \right]^t U \left[I - \gamma \Lambda \right]^t \right\|_2 \\
        = &\max_{j \in [n]} (1 - \gamma \Lambda_{jj})^{2t} \Lambda_{jj} \\
        \leq & \frac{1}{2e \gamma t},
    \end{aligned}
\end{equation}
where $\Lambda_{jj}$ is the $j$-th eigenvalue in the diagonal of $\Lambda$.

Substituting \eqref{eq.proof.iteration} to \eqref{eq.proof.gd_rf}, we have
\begin{equation}
    \begin{aligned}
        \label{eq.rfgd_rf}
        &\|\frrgd - \frrrf\|^2_\Ltwo \\
        \leq &\|\phi(\xx)\|^2 \left\|\phi(\XX)^\top \left(I - \frac{\gamma}{n} \mathbf{K}_M \right)^t\right\|^2 \|\yy\|_{K^{-1}} \\
        \leq &\frac{\kappa^2}{2e \gamma t} \|\yy\|_{K^{-1}}.
    \end{aligned}
\end{equation}
Here, $\|\yy\|_{K^{-1}}$ denote the inverse kernel norm of the labels defined as $\|\yy\|_{K^{-1}} = \yy^\top (K(\XX, \XX))^{-1} \yy$.
According to \cite{scholkopf1998nonlinear}, it holds $\|\yy\|_{K^{-1}} \leq \|\frho\|_K$ under the assumption that the true regression $\frho(\xx) = \langle \frho, K_{\xx} \rangle$ lies in the RKSH of the kernel $K$, i.e., $\frho \in \mh$.

Substituting \eqref{eq.rfsgd_rfgd} and \eqref{eq.rfgd_rf} to \eqref{eq.rf_rfsgd.decomposition}, we prove the result.
\end{proof}

\begin{proof}[Proof of Theorem \ref{thm.rf_overparameterized}]
    The Ridgeless RF predictor admits the bias-variance decomposition
    \begin{align*}
        \mathbb{E} \mE(\frrrf) = \mE\left[\mathbb{E}~ (\frrrf)\right] + \mathbb{E} \left[\text{Var}(\frrrf)\right].
    \end{align*}

    And thus, the random features error have
    \begin{align}
        \label{eq.decomposition.rf_error}
        \mathbb{E} \mE(\frrrf) - \mE(\fkrr) = \left\|\mathbb{E}~ (\frrrf) - \fkrr\right\|_\Ltwo^2 + \mathbb{E}\left[\text{Var}(\frrrf)\right].
    \end{align}
    
    We estimate these two terms, respectively.
    From Proposition 3.1, one can obtain 
    \begin{align}
        \label{eq.rf_unbiased}
        \mathbb{E} ~ \big[\frrrf(\xx)\big] = K(\xx, \XX) K(\XX, \XX)^\dagger \yy = \fkrr(\xx).
    \end{align}
    
    Thus, we only need to bound the variance in the overparameterized case.
    From Proposition C.15 \cite{jacot2020implicit}, the variance of ridgeless RF predictor can be bounded by
    \begin{align}
        \label{eq.variance.overparameterized}
        \mathbb{E}\left[\text{Var}(\frrrf)\right] \leq \frac{(\alpha + c_1) \kappa \|\yy\|^2_{K^{-1}}}{M},
    \end{align}
    where $\alpha$ is some constant related to $M/n$.
    From Proposition C.11 and $\tilde{\lambda} = 0$, we obtain $\alpha = \frac{M/n}{M/n-1}$ when $M \geq n$.

    By substituting \eqref{eq.rf_unbiased} and \eqref{eq.variance.overparameterized} to \eqref{eq.decomposition.rf_error}, we prove the result.
\end{proof}

\begin{proof}[Proof of Theorem \ref{thm.rf_underparameterized}]
    Theorem 4.1 in \cite{jacot2020implicit} provided the approximate between the ridge RF estimator and the ridge regression estimator, we make the RF estimator to the ridgeless case.
    Theorem 4.1 provided that
    \begin{align}
        \label{eq.proof.rf_approximation}
        |\mathbb{E} \widehat{f}_{M, \tilde{\lambda}}(\xx) - \fridgekrr(\xx)| \leq \frac{c\sqrt{K(\xx, \xx)}\|\yy\|_{K^{-1}}}{M}
    \end{align}
    with the effective ridge $\tilde{\lambda}$ satisfying
    \begin{align*}
        \lambda = \tilde{\lambda} + \frac{{\lambda}}{M/n} \sum_{i=1}^n \frac{d_i}{{\lambda} + d_i}
    \end{align*}
    where $d_i$ is the $j$-th element on diagonal form of $\mathbf{K}$.
    We consider the ridgeless RF case with $\tilde{\lambda} = 0$, such that
    \begin{align*}
        \frac{M}{n} = \frac{1}{n} \sum_{i=1}^n \frac{d_i}{\lambda + d_i} = \frac{1}{n} \text{Tr}\left[\mathbf{K}(\mathbf{K} + \lambda I)^{-1}\right].
    \end{align*}

    From the proof of Proposition C.7 in \cite{jacot2020implicit}, the constant is
    $c = \left(\frac{\text{Tr}(\mathbf{K})}{M}\right)^2$.
    Using Assumption \ref{asm.rf}, we have $\sqrt{K(\xx, \xx)} \leq \kappa$
    Substituting the estimate for $\lambda, c, \sqrt{K(\xx, \xx)}$ to \eqref{eq.proof.rf_approximation}, we obtain
    \begin{align}
        \label{eq.rf_error.bias}
        \left\|\mathbb{E}[\frrrf] - \fridgekrr\right\|_\Ltwo^2 \leq \frac{\kappa^2 \text{Tr}(\mathbf{K})^4 \|\yy\|_{K^{-1}}^2}{M^6}.
    \end{align}

    From Proposition C.15 \cite{jacot2020implicit}, the variance of ridgeless RF predictor can be bounded by
    \begin{align}
        \label{eq.variance.underparameterized}
        \mathbb{E}\left[\text{Var}(\frrrf)\right] \leq \frac{(\alpha + c_1) \kappa \|\yy\|^2_{K^{-1}}}{M},
    \end{align}
    where $\alpha$ is related to $M/n$.
    From Proposition C.11, setting $\tilde{\lambda} = 0$, we obtain $\alpha = \frac{1}{1-M/n}$ when $M < n$.

    By substituting \eqref{eq.rf_error.bias} and \eqref{eq.variance.underparameterized} to \eqref{eq.decomposition.rf_error}, we prove the result.
\end{proof}

\end{document}